\documentclass[]{style/iromlab}

\DeclareDocumentEnvironment{example}{}{\noindent\textbf{Running example:}\itshape}{}

\usepackage{ifthen}
\newboolean{include-notes}
\newboolean{include-new}
\newboolean{include-remove}
\setboolean{include-notes}{true}
\setboolean{include-new}{false}
\setboolean{include-remove}{false}

\usepackage[dvipsnames]{xcolor}
\usepackage[normalem]{ulem}
\newcommand{\justin}[1]{\ifthenelse{\boolean{include-notes}}{\textcolor{orange}{\textbf{Jaime:} #1}}{}}

\newcommand{\princeton}[1]{\ifthenelse{\boolean{include-notes}}{\textcolor{orange}{#1}}{}}

\usepackage{multicol}
\usepackage{amsmath, amsfonts, amssymb, amsthm}
\usepackage{enumerate}
\usepackage[inline]{enumitem}
\usepackage{mathtools}
\usepackage{graphicx}
\usepackage{longtable,tabularx}
\usepackage{placeins} 
\usepackage{float}
\usepackage{multirow}
\usepackage{bbm}
\usepackage{threeparttable}
\usepackage{balance}
\usepackage{algorithm}
\usepackage{algpseudocode}
\usepackage{adjustbox}
\usepackage{booktabs}
\usepackage{bm}
\usepackage{etoolbox}
\usepackage{microtype}
\usepackage[title]{appendix}
\usepackage{units}
\usepackage{cleveref}
\usepackage{xspace}
\usepackage{tcolorbox}
\usepackage{caption}
\usepackage{wrapfig}

\definecolor{claude_color}{HTML}{F89E62}
\definecolor{deepseek_color}{HTML}{78B6E8}
\definecolor{o3_mini_color}{HTML}{6CD5A1}
\definecolor{red_color}{HTML}{E13B55}
\definecolor{box_color}{HTML}{f7efe9}

\tcbset{
  promptstyle/.style={
    colback=gray!10,
    colframe=gray!60,
    boxrule=0.5pt,
    arc=2pt,
    outer arc=2pt,
    left=4pt,
    right=4pt,
    top=4pt,
    bottom=4pt,
    fonttitle=\bfseries,
    sharp corners=south,
  }
}

\tcbset{
  box_style/.style={
    colback=box_color!40,
    colframe=box_color!80,
    coltitle=black,
    boxrule=1.0pt,
    arc=2pt,
    outer arc=2pt,
    left=4pt,
    right=4pt,
    top=-10pt,
    bottom=4pt,
    fonttitle=\bfseries,
  }
}

\newtheorem{proposition}{Proposition}
\newtheorem{definition}{Definition}
\newtheorem{corollary}{Corollary}

\usepackage{xcolor}
\definecolor{aniRed}{HTML}{D00000}

\newbool{extended}
\setbool{extended}{false}

\makeatletter
\newcommand{\longdash}[1][2em]{%
  \makebox[#1]{$\m@th\smash-\mkern-7mu\cleaders\hbox{$\mkern-2mu\smash-\mkern-2mu$}\hfill\mkern-7mu\smash-$}}
\makeatother
\newcommand{\omitskip}{\kern-\arraycolsep}

\author[]{Anirudha Majumdar}

\affiliation[]{Princeton University}
\contribution[]{\texttt{ani.majumdar$@$princeton.edu}}

\begin{document}

\title{{\LARGE Deceptive Risk Minimization: Out-of-Distribution Generalization by Deceiving Distribution Shift Detectors}}

\abstract{
This paper proposes \emph{deception} as a mechanism for out-of-distribution (OOD) generalization: by learning data representations that make training data \emph{appear} independent and identically distributed (iid) to an observer, we can identify stable features that eliminate spurious correlations and generalize to unseen domains. We refer to this principle as \emph{deceptive risk minimization} (DRM) and instantiate it with a practical differentiable objective that simultaneously learns features that eliminate distribution shifts from the perspective of a detector based on conformal martingales while minimizing a task-specific loss. In contrast to domain adaptation or prior invariant representation learning methods, DRM does not require access to test data or a partitioning of training data into a finite number of data-generating domains. We demonstrate the efficacy of DRM on numerical experiments with concept shift and a simulated imitation learning setting with covariate shift in environments that a robot is deployed in. 
}





\metadata[Project Website]{\href{https://deceptive-risk.github.io}{deceptive-risk.github.io}}
\metadata[Code]{\href{https://github.com/irom-princeton/deceptive-risk-minimization}{github.com/irom-princeton/deceptive-risk-minimization}}

\maketitle


\section{Introduction}
\label{sec:intro}

Is there an unbridgeable gap between in-distribution (ID) and out-of-distribution (OOD) generalization in machine learning? Or can the distinction be erased by a change in perspective? Traditional wisdom holds that there is a vast chasm between the two settings. 
Applications where training environments are representative of test environments (e.g., via careful curation of large-scale datasets) have seen remarkable empirical progress and real-world impact. This success is backed by a deep theoretical understanding of ID generalization from decades of progress in statistical learning theory~\cite{shalev2014understanding}. However, in settings where it is challenging to cover all relevant dimensions of variation exhaustively in the training data --- a common occurrence in real-world applications such as robotics, healthcare, and cybersecurity --- high-capacity models can absorb spurious correlations and fail catastrophically in test settings where these correlations are altered or even reversed~\cite{liu2021towards, sinha2022system, li2025probing, arjovsky2019invariant}.  

In this paper, we take an \emph{observer-centric} viewpoint on the gap between ID and OOD generalization. Our starting point is the following basic observation: from the perspective of an external observer who cannot discern distribution shifts in a sequence of data, \emph{OOD generalization is equivalent to ID generalization}. As an example, consider an observer responsible for overseeing the performance of a robot operating in a warehouse. The robot is presented with a sequence of objects to place into a receptacle, while the observer records the corresponding sequence of bits denoting success (1) or failure (0) of the robot on each object. During this process, the robot encounters changes to its lighting conditions, appearances of objects, and its visual backdrop. However, the robot is able to maintain reliable performance throughout these changes, with only a small-but-consistent failure probability. As a result, the observer is completely oblivious to the distribution shifts faced by the robot: the data recorded by the observer has shed its spurious, domain-specific cues and appears independent and identically distributed (iid). In a sense, the robot has \emph{hidden} the distribution shifts from the observer. 


The core idea of this work is to translate this observer-centric perspective on generalization into a prescriptive mechanism for OOD generalization. Suppose that a learner is presented with a sequence of training data that exhibits (potentially mild) distribution shifts. Then, by learning data representations that eliminate these distribution shifts from the perspective of an observer, we can identify stable features that do not rely on spurious correlations and generalize to unseen domains. 
For example, if the robot encounters periodic changes in lighting conditions or visual backdrops (Fig.~\ref{fig:anchor}), an encoding of observations that hides these changes from a distribution shift detector will eliminate sensitivity to the changes and result in robust performance when spurious correlations in training data are significantly exaggerated or reversed. 

Concretely, we formulate this learning mechanism as an adversarial game (Fig.~\ref{fig:anchor}) which we refer to as \emph{deceptive risk minimization} (DRM). An encoder network learns to generate representations that support minimization of a task-specific loss while simultaneously eliminating distribution shifts from the perspective of a detector presented with a sequence of learned representations. We assume that this sequence is presented in the order in which training data were curated, and thus preserves the structure of natural distribution shifts in the original data. For example, in robotics, training data are often collected in environments that vary over time either discretely (e.g., a change in the color of the table) or continuously (e.g., a continuous change in ambient outdoor lighting over a day). We are interested in settings where the training data sequence exhibits some distribution shifts, which are exaggerated or reversed at deployment time. Importantly, unlike prior work in domain adaptation~\cite{ben2010theory, ganin2016domain, ganin2015unsupervised, zhang2015multi, long2018conditional, gong2016domain, li2018deep, courty2016optimal} or invariant representation learning~\cite{arjovsky2019invariant, peters2016causal, ahuja2020invariant, krueger2021out}, we do not assume access to any data from test environments or that training data are partitioned into different domains corresponding to data-generating distributions; we simply assume that the order of data is preserved. Associating data points with a finite set of domains --- either manually or via unsupervised clustering~\cite{le2025invariant, murata2025clustered} --- is often impractical or unachievable in settings where there is a continuous change in conditions.

\begin{figure}
    \centering
    \includegraphics[width=1\linewidth]{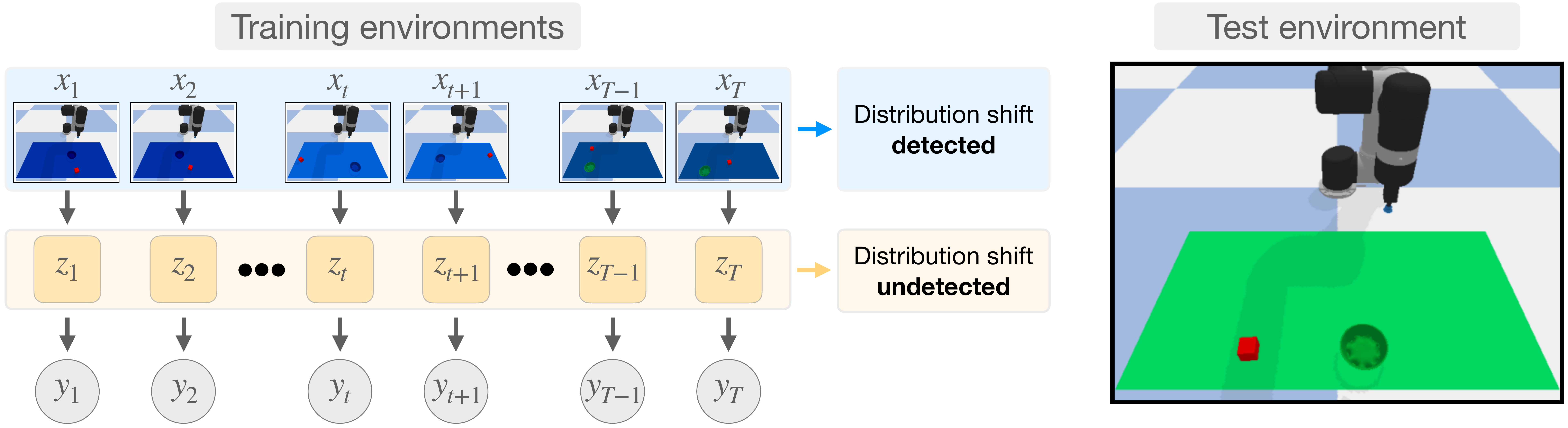}
    \caption{Deceptive risk minimization (DRM): by learning data representations that make training data \emph{appear} independent and identically distributed to an observer (left), we can identify stable features that eliminate spurious correlations and generalize to unseen domains (right). In the figure, the robot's training environments undergo small-but-structured changes in the colors of the table and objects; DRM learns a representation that is insensitive to these changes, which results in generalization to environments with significantly different appearances.}
    \label{fig:anchor}
\end{figure}

We present a practical instantiation of DRM that utilizes \emph{conformal martingales} (CMs)~\cite{vovk2005algorithmic, vovk2003testing, vovk2021testing} for distribution shift detection. CMs offer a general and flexible approach to distribution shift detection, which is often highly effective in practical scenarios~\cite[Ch. 8]{vovk2005algorithmic}. Concretely, the CM approach computes a quantity that remains small when data are iid (or exchangeable), but that can grow quickly in the presence of distribution shifts. We derive an end-to-end differentiable loss that penalizes the conformal martingale computed on encoded inputs; this loss serves to train the encoder to learn representations that eliminate distribution shifts from the perspective of the CM-based detector. 

{\bf Summary of contributions.} We introduce deceptive risk minimization (DRM): a novel learning principle that estimates representations that eliminate spurious correlations by deceiving distribution shift detectors. 
We develop a practical instantiation of DRM via a differentiable loss that penalizes conformal martingales, and demonstrate the efficacy of this representation learning objective in different empirical settings involving covariate and concept shift. Conceptually, DRM creates a bridge between distribution shift detection and OOD generalization, which we hope future work can build on to unlock practical methods for OOD generalization in real-world applications. 

\newpage
\section{Prelude: randomness and structure in the eye of the beholder}
\label{sec:prelude}


We begin with a philosophical perspective that will shape our algorithmic approach to OOD generalization. Consider the following sequence of digits:
\begin{figure}[H]
\vspace{-7pt}
    \centering
    \includegraphics[width=1\linewidth]{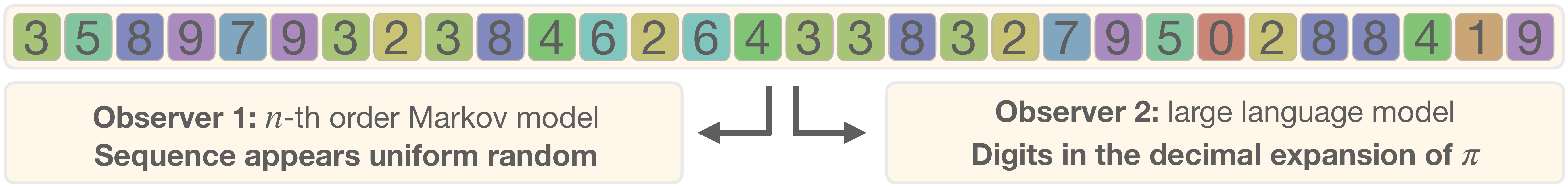}
\vspace{-15pt}
\end{figure}

An observer who fits an $n$-th order Markov model to this sequence observes no discernible pattern; the digits appear randomly and uniformly distributed conditioned on past inputs. However, an observer in the form of a large language model identifies the sequence as being \emph{deterministically} generated; the sequence forms the decimal expansion of $\pi$ from the tenth place onward.

This simple example illustrates the \emph{observer-centric} nature of randomness. A given observer may discern deterministic patterns in data that are imperceptible to a different observer. This observer-centric (or \emph{subjective}) view has spurred decades of thinking on the nature of randomness~\cite{gillies2012philosophical}. Does randomness have an objective existence in the physical world, or is all randomness subjective~\cite{de2017theory}? The observer-centric perspective also informs the notion of \emph{pseudorandmness}~\cite{vadhan2012pseudorandomness}, which plays a central role in crypography and theoretical computer science more broadly. Pseudorandom sequences are data that ``appear random" in the sense that a computationally efficient observer cannot distinguish such sequences from random sequences. 

In this paper, we adopt an observer-centric viewpoint on the bedrock assumption of machine learning theory and practice --- that data are independent and identically distributed (iid)\footnote{A single data sequence is not sufficient to distinguish between iid random variables and exchangeable random variables~\cite{ramdas2022testing}; as a result, in order to keep the exposition simple, we do not distinguish between the two here.} --- and use it to formulate a new mechanism for OOD generalization. The key idea is to learn representations that make data \emph{appear} iid from the perspective of an observer. We provide an illustrative example below in order to explain the key intuitions.  

\subsection{An illustrative example}
\label{sec:illustrative example}

\begin{wrapfigure}{r}{0.6\textwidth}
    \centering
    \includegraphics[width=0.6\textwidth]{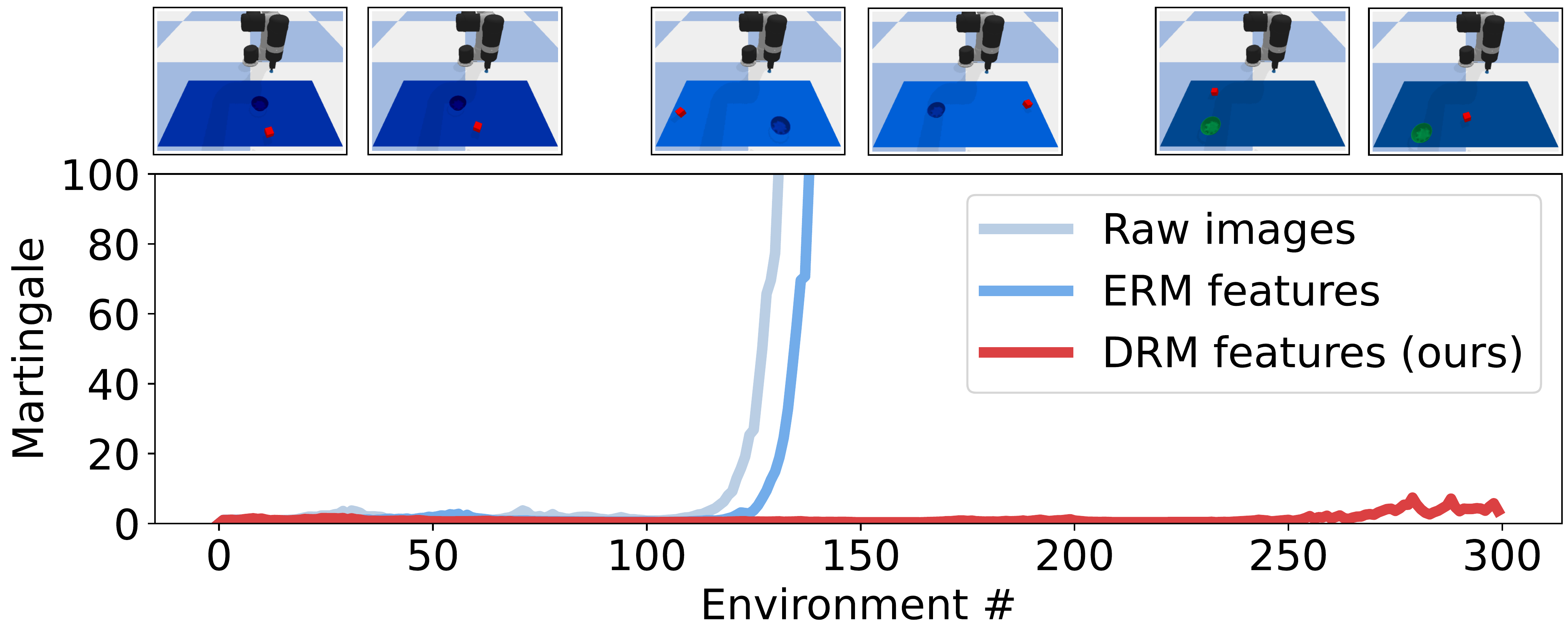} 
    \caption{A detector based on conformal martingales rapidly detects the distribution shift given the sequence of raw images or features computed via ERM; in contrast, features from DRM appear iid. \label{fig:martingales}}
    \label{fig:wrapped}
    \vspace{-10pt}
\end{wrapfigure}
Consider an imitation learning setting where a human has provided a sequence of examples to a robot demonstrating how to perform a given task (Fig.~\ref{fig:anchor}). These demonstrations are provided in environments that exhibit a small amount of distribution shift. Specifically, the color of the table is varied \emph{slightly} in a structured way: a third of the demonstrations are provided with one table-bowl color combination, the next third with a slightly different color combination, and the final third with another. The standard approach to learning a policy in such a setting is empirical risk minimization (ERM): \emph{assume} that data are iid and learn a mapping from the robot's observations to actions by minimizing a behavior cloning loss on training data. Such a policy performs well when deployed with table colors similar to ones seen during training, but fails with colors that are significantly different (see Sec.~\ref{sec:imitation learning} for numerical results). 

The starting point for our approach is the observation that the non-iid nature of data in this setting can be \emph{inferred from the sequence of training environments}. Fig.~\ref{fig:martingales} shows the output from a distribution shift detector based on conformal martingales (described formally in Sec.~\ref{sec:conformal martingales}) computed on observations from the training environments. This detector spikes strongly once the table color is changed. Crucially, the detector also spikes when provided with the sequence of latent features from the policy computed via ERM, indicating that the policy's latent representation encodes color information. 

Now consider a policy that eliminates the distribution shift from the perspective of an observer who is only presented with latent representations from the policy. Intuitively, the data can be made to ``appear iid" by eliminating sensitivity to the table color, which leads to OOD generalization to different table colors. 

\section{Problem formulation}
\label{sec:problem formulation}

{\bf Training data.} We begin by formalizing the supervised learning setting in which we consider OOD generalization. A learner is presented with a \emph{sequence} of training data $((x_t, y_t))_{t=1}^T \in (\mathcal{X} \times \mathcal{Y})^T$ consisting of input-label pairs sampled from a sequence of random variables $((X_t, Y_t))_{t=1}^T$, which may be dependent and non-identically distributed. This sequential collection of data is a core assumption of our work and departs from the standard practice of shuffling data that has been collected~\cite{arjovsky2019invariant}. As an example, data for imitation learning in robotics is often collected sequentially in multiple environments over multiple months~\cite{team2025gemini, intelligence2025pi}. If some of the data collection is parallelized (e.g., by multiple human operators collecting data on different robots on the same day), we assume that this data is serialized by committing to a particular ordering, e.g., $(<$\texttt{Data from operator 1 on day 1}$>$, $<$\texttt{Data from operator 2 on day 1}$>$, \dots,  $<$\texttt{Data from operator 1 on day 2}$>$, $<$\texttt{Data from operator 2 on day 2}$>$, \dots $)$.

{\bf Hypothesis and loss.} Given the training data, the learner produces a hypothesis $h: \mathcal{X} \rightarrow \mathcal{Y}$, which maps inputs to predicted labels. The hypothesis is evaluated according to a loss function $l: \mathcal{Y} \times \mathcal{Y} \rightarrow \mathbb{R}$, which compares predicted labels with ground-truth labels.

{\bf OOD generalization.} The learned hypothesis $h$ is deployed on a sequence of test data drawn from random variables $((X_\tau, Y_\tau))_{\tau=T+1}^{T+T'}$. The overall quality of the hypothesis is measured by its expected loss on test data:
\begin{equation}
    \frac{1}{T'} \sum_{\tau=T+1}^{T+T'} \ \mathbb{E} \ \Big{[}l(h(X_\tau), Y_\tau)\Big{]}.
\end{equation}
In the case where the random variables $(X_t, Y_t)_{t=1}^{T+T'}$ are iid, the formulation above reduces to the standard setting of statistical learning theory. Of interest to us is the case where the training data sequence reflects some (possibly mild) distribution shifts, which are significantly exaggerated or reversed for test data. 

{\bf Covariate and concept shifts.} We consider two types of distribution shifts: (i) covariate shift, where the distribution of inputs changes over time (e.g., changes in the color of the robot's table in the example from Sec.~\ref{sec:illustrative example}), and (ii) (anti-causal) concept shift, where the conditional distribution of $X|Y$ changes over time (e.g., an image classification setting where the appearances of images for a given label change over time). We discuss potential extensions to other types of shifts (e.g., causal concept shift) in Sec.~\ref{sec:discussion}.  
\section{Algorithmic implementation: deceptive risk minimization}
\label{sec:algorithm}

Our goal is to find features that eliminate the distribution shift between training and test settings. Since the learner is only provided with the sequence of training data, we utilize distribution shifts observed in this data as a proxy. Specifically, we learn features that are stable along the training data sequence --- in the sense that they appear iid to an observer --- while also supporting the minimization of the task-specific loss. We formalize this objective by defining an observer in the form of a distribution shift detector.


\subsection{Defining the observer: distribution shift detection}
\label{sec:observer}

Let $(X_1, X_2, \dots)$ be a sequence of input random variables, and let $\phi$ be a mapping from an input $x$ to an encoding $\phi(x) \in \mathbb{R}^{n_d}$. We define an observer $\Delta$ who takes as input a realization $(\phi(x_1), \phi(x_2), \dots)$ from the sequence of random variables $(\phi(X_1), \phi(X_2), \dots)$ and outputs a boolean $\delta \in \{\text{True},\text{False}\}$ indicating if a deviation from the iid hypothesis has been detected. 

\begin{definition}[Practically iid]
A sequence $(\phi(x_1), \phi(x_2), \dots)$ is $\Delta$-\emph{practically iid} if a distribution shift detector $\Delta$ does not trigger\footnote{To keep this definition simple, we restrict attention to deterministic detectors (e.g., constructed by taking a detector that has randomness and fixing the seed).} when provided with this sequence as input, i.e., $\Delta((\phi(x_1), \phi(x_2), \dots)) = \text{False}$.
\end{definition}

{\bf False alarm rate (FAR).} For a detector $\Delta$ to be useful, it should not trigger too often when presented with data drawn from a sequence of iid random variables. This is captured by the false alarm rate (FAR), which is the worst-case probability of detection when data are drawn from an iid sequence of random variables~\cite[Ch. 8]{vovk2005algorithmic}. There are a number of methods for constructing detectors that have a bounded FAR (Sec.~\ref{sec:related work}); in this work, we will specifically utilize conformal martingales~\cite{vovk2021testing, vovk2005algorithmic}. 

\subsection{Formalizing the DRM objective} 

Next, we formulate the objective of deceptive risk minimization (DRM) as a constrained optimization problem. We consider hypotheses $h: x \mapsto \phi(x) \mapsto f(\phi(x)) \in \mathcal{Y}$, which encode inputs using $\phi$ and map these to labels via $f$. We use the notation $(\phi(x_t) | y)_{t=1}^T$ to denote the subsequence of $(\phi(x_t))_{t=1}^T$ with labels equal to $y$. The following optimization problem minimizes the task-specific loss $l$ while searching for a representation that makes the training data $\Delta$-practically iid. Our optimization problem considers two types of constraints, which are aimed at tackling covariate shift and concept shift respectively. 

\begin{tcolorbox}[box_style, title=Deceptive risk minimization (DRM)]
\begin{align}
\inf_{f, \phi} \ \ &\frac{1}{T} \sum_{t=1}^T l(x_t, f(\phi(x_t)) \nonumber \\
\text{s.t.} \ \ &(\phi(x_t))_{t=1}^T \ \text{is} \ \Delta\text{-practically iid} \quad \text{[covariate shift]} \label{eq:covariate shift constraint} \\
&(\phi(x_t) | y)_{t=1}^T \ \text{is} \ \Delta\text{-practically iid}, \forall y \in \mathcal{Y} \quad \text{[concept shift]}. \label{eq:concept shift constraint}
\end{align}
\end{tcolorbox}

\subsection{Instantiation with conformal martingales}
\label{sec:conformal martingales}

We now describe a particular distribution shift detector $\Delta$ based on conformal martingales (CMs). This will allow us to flexibly handle both covariate and concept shifts. In addition, this detector will allow us to formulate a \emph{differentiable} surrogate for the constraints \eqref{eq:covariate shift constraint} and \eqref{eq:concept shift constraint} in the DRM optimization problem.

Intuitively, the CM approach constructs a quantity that grows quickly when random variables are not iid, and remains small otherwise. In order to detect covariate shift, we first assess how well every encoded data point $\phi(x_i)$ (with $i \leq t$) \emph{conforms} to the sequence of data points $(\phi(x_{j}))_{j=1}^{t}$ observed up to time $t$ using a conformity score:
\begin{equation}
    \alpha_i^\text{covariate} := \min_{j \in \{1,\dots,t\}: j \neq i} \ \ d(\phi(x_i), \phi(x_j)), \label{eq:conformity score covariate}
\end{equation}
where $d: \mathbb{R}^{n_d} \times \mathbb{R}^{n_d} \rightarrow \mathbb{R}^+$ captures how different two encoded data points $\phi(x_i)$ and $\phi(x_j)$ are. In our numerical experiments, we will use the cosine distance or its sharpened form~\cite{ahmad2024scsnet}:
\begin{equation}
    d(z, z') := 1 - \text{sign}(z \cdot z') \left| \frac{z \cdot z'}{\|z\|\|z'\|} \right|^\gamma,
\end{equation}
where $\gamma=1$ (standard cosine distance) or $\gamma=2$. In cases where we are interested in detecting concept shift rather than covariate shift, we will alternately use a label-conditioned conformity score:
\begin{equation}
    \alpha_i^\text{concept} := \min_{j \in \{1,\dots,t\}: j \neq i, y_j = y_i} \ \ d(\phi(x_i), \phi(x_j)). \label{eq:conformity score concept}
\end{equation}
The conformity scores are used to compute conformal p-values for each $t \leq T$:
\begin{align}
    &p_t^\text{covariate} := \frac{|\{i | 1 \leq i \leq t,  \alpha_i < \alpha_t\}| + \xi_t | \{i | 1 \leq i \leq t, \alpha_i = \alpha_t\}|}{t}, \label{eq:p-values-covariate} \\
    &p_t^\text{concept} := \frac{|\{i | 1 \leq i \leq t,  \alpha_i < \alpha_t, y_i=y_t\}| + \xi_t | \{i | 1 \leq i \leq t, \alpha_i = \alpha_t, y_i=y_t\}|}{|\{i | 1 \leq i \leq t, y_i=y_t\}|}, \label{eq:p-values-concept}
\end{align}
where $\xi_t \in [0,1]$ is sampled independently from the uniform distribution on $[0,1]$. 

As shown in \cite[Ch. 2]{vovk2005algorithmic}, the p-values $p_t^\text{covariate}$ and $p_t^\text{concept}$ are independent and uniformly distributed in $[0,1]$ in the \emph{absence} of covariate and concept shift respectively. Thus, the CM approach constructs a quantity that measures how far away from being uniformly and independently distributed the p-values are. This is achieved using a betting martingale~\cite[Ch. 8]{vovk2005algorithmic}, whose computation is shown in Algorithm~\ref{alg:betting martingale}. 
\begin{algorithm}[H]
\SetAlgoLined
{\bf Inputs: } p-values $p_1, p_2, \dots, p_T$ (computed using \eqref{eq:p-values-covariate} or \eqref{eq:p-values-concept}) \\
{\bf Outputs: } martingale values $S_1, S_2, \dots, S_T$ \\
Define constants $E := \{-1, -0.5, 0, +0.5, +1\}$, \ $\mu := 0.005$ \\
Initialize $C \leftarrow 1$; $C_e \leftarrow 1/|E|, \ \forall e \in E$ \\
 \For{$t=1$ to $T$}{
  {\bf for} $e \in E$ {\bf do}: $C_e \leftarrow (1-\mu)C_e + (\mu/|E|)C$ \\
  {\bf for} $e \in E$ {\bf do}: $C_e \leftarrow C_e [1 + e(p_t - 0.5)]$ \\
  $C \leftarrow \sum_{e \in E} \ C_e$ \\
  $S_t \leftarrow C$
 }
 \caption{Betting martingale~\cite[Ch. 8]{vovk2005algorithmic} \label{alg:betting martingale}}
\end{algorithm}
Intuitively, the betting martingale represents the capital of a bettor who gambles \emph{against} the hypothesis that random variables are iid.
Large values of the betting martingale values $S_t$ thus serve as an indicator of distribution shift. Conversely, in the \emph{absence} of distribution shift, $S_t$ is guaranteed to remain small with high probability. 

\begin{proposition}[False alarm control~\cite{vovk2005algorithmic}] 
Consider a detector $\Delta^\text{CM}$ which is triggered if the betting martingale ever exceeds a threshold $1/\alpha$. This detector controls the false alarm rate to level $\alpha$, i.e., the probability that $\Delta^\text{CM}$ triggers when the random variables $(\Phi(X_t))_{t=1}^T$ (or $(\Phi(X_t) | y)_{t=1}^T$ in the case of concept shift) are iid is upper bounded by $\alpha$. 
\end{proposition}

\subsection{Making the objective differentiable}
\label{sec:differentiable objective}

Conformal martingales suggest a practical method to instantiate the DRM optimization problem: replace the hard constraints \eqref{eq:covariate shift constraint} and \eqref{eq:concept shift constraint} with soft constraints that penalize large values of the betting martingale. The only remaining hurdle is to make the martingale computation differentiable. The steps in Algorithm~\ref{alg:betting martingale} are differentiable, and hence the only sources of non-differentiability are in the computation of the conformity scores $\alpha_i$ (\eqref{eq:conformity score covariate} or \eqref{eq:conformity score concept}) and the p-values $p_t$ (\eqref{eq:p-values-covariate} or \eqref{eq:p-values-concept}). We follow a procedure similar to~\cite{stutz2021learning}, which differentiates through the calibration procedure of conformal prediction. We replace the minimization operation in \eqref{eq:conformity score covariate} or \eqref{eq:conformity score concept} by the standard soft-min operation. The computation of $\{i | 1 \leq i \leq t,  \alpha_i < \alpha_t\}$ (or $\{i | 1 \leq i \leq t,  \alpha_i < \alpha_t, y_i=y_t\}$) is equivalent to the computation of a quantile. This can be approximated by smoothed sorting methods~\cite{blondel2020fast, cuturi2019differentiable}, which have a ``dispersion" hyperparameter $\sigma$ such that smooth sorting approaches hard sorting as $\sigma \rightarrow 0$. 

We thus formulate the practical instantiation of DRM as follows by optimizing a weighted combination of the task-specific supervised learning loss (e.g., cross-entropy) and the soft martingale values. 

\vspace{10pt}

\begin{tcolorbox}[box_style, title=Deceptive risk minimization: differentiable objective]
\begin{align}
&\inf_{f, \phi} \ \ \frac{1}{T} \sum_{t=1}^T l(x_t, f(\phi(x_t)) + \lambda \frac{1}{T} \sum_{t=1}^T \tilde{S}_t(\phi(x_1), \dots, \phi(x_t)), \label{eq:drm objective differentiable} \\
&\text{where $\tilde{S}_t(\phi(x_1), \dots, \phi(x_t))$ is computed via Algorithm~\ref{alg:differentiable drm}. }
\end{align}
\end{tcolorbox}

\begin{algorithm}[H]
\SetAlgoLined
{\bf Inputs: } sequence of features $(\phi(x_t))_{t=1}^T$ \\
{\bf Outputs: } sequence of soft martingale values $(\tilde{S}_t)_{t=1}^T$ \\
 \For{$t=1$ to $T$}{
 \For{$i=1$ to $t$}{
  Compute $\tilde{\alpha_i} \leftarrow \text{soft}(\alpha_i^\text{covariate})$ by replacing min with soft-min in \eqref{eq:conformity score covariate} [covariate shift] \\
  or $\tilde{\alpha_i} \leftarrow \text{soft}(\alpha_i^\text{concept})$ by replacing min with soft-min in \eqref{eq:conformity score concept} [concept shift]
  }
  Compute $\tilde{p}_t^\text{covariate} \leftarrow \text{soft}(p_t^\text{covariate})$ using \eqref{eq:p-values-covariate} with soft-quantile [covariate shift] \\
  or $\tilde{p}_t^\text{concept} \leftarrow \text{soft}(p_t^\text{concept})$ using \eqref{eq:p-values-concept} with soft-quantile [concept shift] \\
 } 
 Compute $(\tilde{S}_t)_{t=1}^T$ using Algorithm~\ref{alg:betting martingale} with inputs $(\tilde{p}_t^\text{covariate})_{t=1}^T$ or $(\tilde{p}_t^\text{concept})_{t=1}^T$.
 \caption{Computing the DRM regularizer \label{alg:differentiable drm}}
\end{algorithm}

\subsection{Algorithmic implementation details}
\label{sec:implementation details}

We end our exposition of the DRM algorithm by discussing a few implementation details. 

{\bf Multiple detection sequences.} In settings where the training data sequence is large, we subsample multiple sequences, compute (soft) martingales for each, and average these to form the regularization term in Eq.~\eqref{eq:drm objective differentiable}. This results in improved computational efficiency and robustness compared to computing a single martingale value from the entire training data sequence. 

{\bf Feature normalization.} As described in Sec.~\ref{sec:conformal martingales},  we utilize cosine distances to define conformity scores. Since the resulting conformity scores are only sensitive to directional differences between features, we normalize encodings to have unit norm, i.e., $\|\phi(x)\|_2 = 1$. 

{\bf Warm-starting with ERM.} For the Colored-MNIST example (Sec.~\ref{sec:colored mnist}), we found that warm-starting DRM with a small number of epochs of ERM helped improve performance. This is consistent with the implementation of invariant risk minimization (IRM) in \cite{arjovsky2019invariant}.

\section{Experiments}
\label{sec:experiments}

We evaluate DRM in three sets of experiments, which seek to investigate the following questions:
\begin{enumerate}
\item How effective is DRM in enabling OOD generalization in settings that involve concept shift or covariate shift with spurious correlations in the training data? 
\item Can DRM match the performance of invariant risk minimization (IRM)~\cite{arjovsky2019invariant}, which assumes an oracle partitioning of training data into a finite number of ``environments" corresponding to different data distributions? 
\item How effective is the conformal martingale approach for distribution shift detection, which forms the bedrock of DRM's algorithmic implementation? 
\end{enumerate}

Code for experiments can be found at the \href{https://deceptive-risk.github.io/}{project page}, and hyperparameters are listed in Appendix~\ref{app:hyperparameters}.

\subsection{Concept shift: toy 2D example}
\label{sec:toy example}

We begin with a binary classification task with 2D inputs, where one input dimension correlates strongly but spuriously with the label. Empirical risk minimization (ERM) latches on to this correlation and learns a classifier that relies heavily on the spuriously correlated input dimension. However, when the correlation is reversed at test time, the performance of the ERM classifier degrades significantly. This is a 2D version of the Colored-MNIST task introduced in \cite{arjovsky2019invariant}, which allows us to visualize the features learned by DRM. 

{\bf Training and test distributions.} The learner is presented with a sequence of training data $(x_t, y_t)_{t=1}^T$, where $x_t := [x_{t}^{[1]}, x_{t}^{[2]}]$ is two-dimensional and $y_t \in \{0, 1\}$. The first input dimension $x_{t}^{[1]}$ is drawn from a normal distribution $\mathcal{N}(0, 2^2)$, and a preliminary label $\tilde{y}_t$ is assigned purely as a function of $x_{t}^{[1]}$:
\[
\tilde{y}_t = 
\begin{cases}
1 & \text{if } x_{t}^{[1]} \geq 0, \\
0 & \text{otherwise}.
\end{cases}
\]
The final label $y_t$ is assigned by flipping $\tilde{y}_t$ with probability 0.25. The second dimension $x_{t}^{[2]}$ of the input is constructed so that it strongly correlates with the label. Specifically, we first construct
\begin{equation}
    \tilde{x}_t^{[2]} = y_t^\text{sign} + u \cdot y_t^\text{sign}, 
\end{equation}
where $y_t^\text{sign} = 2y_t - 1$ and $u$ is sampled from the uniform distribution on $[0,1]$. If $u=0$, $\tilde{x}_t^{[2]}$ thus perfectly correlates with the label $y_t$. The presence of $u$ ensures that the correlation is strong (but not perfect). The learner observes $x_t^{[2]}$, which flips the sign of $\tilde{x}_t^{[2]}$ with a probability that varies smoothly from $p_1$ to $p_T$ over time: $p_t = p_1 + (p_T - p_1)(t-1)/(T-1)$, for $t \in \{1,\dots,T\}$. We choose $p_1 = 0$ and $p_T = 0.3$.
This time-varying probability is the source of concept shift in the training data, where the distribution of the input conditioned on the label varies slightly over time. At test time, the correlation between $x_t^{[2]}$ and the label $y_t$ is reversed by choosing a flipping probability $p_\text{test} = 0.9$. We highlight that IRM~\cite{arjovsky2019invariant} and its variants~\cite{wang2022generalizing, krueger2021out, ahuja2020invariant, lu2021nonlinear} --- which assume that data are separated into finitely many domains corresponding to different data-generating distributions --- are not directly applicable here since the data-generating distribution changes \emph{continuously} for the training data. 

{\bf Results.} We train a multi-layer perceptron (MLP) using both ERM and DRM, and utilize the last hidden representation (unit-normalized) as our feature $\phi(x_t)$ for computing the conformal test martingale in DRM (Algorithm~\ref{alg:differentiable drm}). 
Fig.~\ref{fig:all results} (left) compares the performance of ERM with DRM on training and test data (across 10 seeds). The reversal of the spurious correlation results in a dramatic drop in performance on test data for ERM. In contrast, the performance of the classifier learned by DRM is almost entirely unimpacted. This performance also nearly matches an oracle that relies exclusively on $x_t^{[1]}$ for classification, which has a $0.75$ classification accuracy on test data (since labels are first generated based on $x_t^{[1]}$ and then flipped with probability 0.25). 

{\bf Visualizing classifiers.} Fig.~\ref{fig:toy example representations} visualizes the classifiers learned by ERM and DRM. ERM learns a classifier that heavily exploits the spuriously correlated input dimension $x_t^{[2]}$ in order to maximize training performance, which leads to a collapse in performance on the test distribution. In contrast, DRM learns a classifier that relies almost exclusively on the robust input dimension $x_t^{[1]}$. 

\begin{figure}
    \centering
    \includegraphics[width=0.99\linewidth]{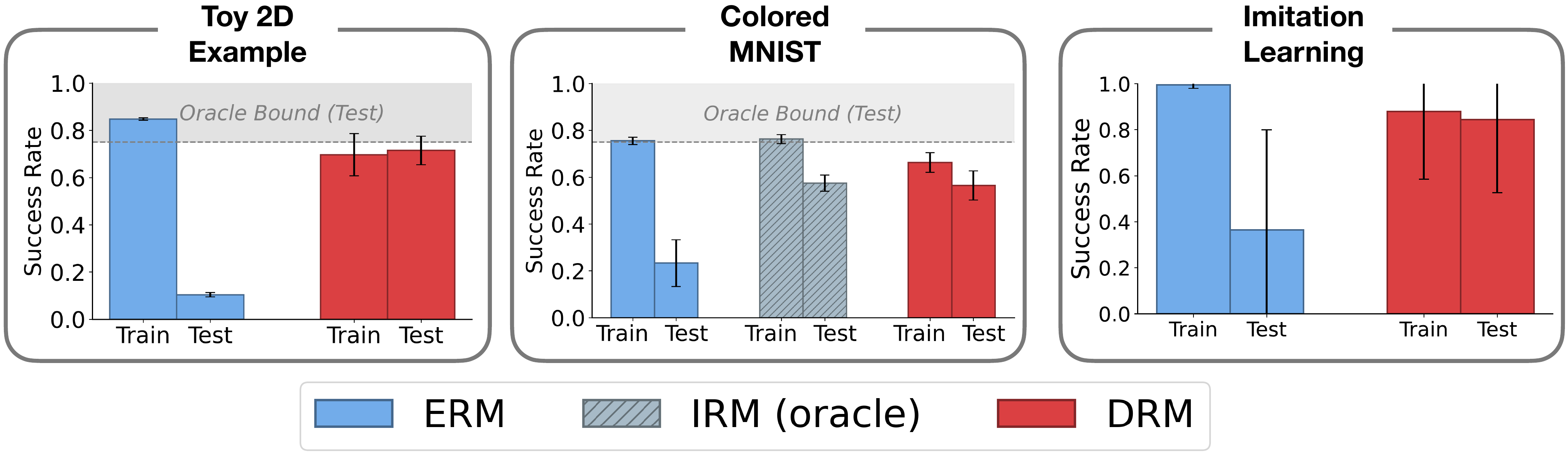}
    \caption{Success rates in our three examples, for training and test settings. ERM latches on to spurious correlations or distractors in each case, which results in severe performance degradation at test time. In contrast, DRM learns stable features that lead to strong generalization. For Colored-MNIST, DRM also matches the performance of IRM, which assumes oracle knowledge of the time at which a distribution shift occurs. IRM is not directly applicable in the 2D example since the distribution shifts \emph{continuously} during training. For the 2D example and Colored-MNIST, the maximal achievable test success for any classifier that sheds spurious correlations is 0.75 (``Oracle Bound (Test)").}
    \label{fig:all results}
\end{figure}

\begin{figure}
    \centering
    \includegraphics[width=0.99\linewidth]{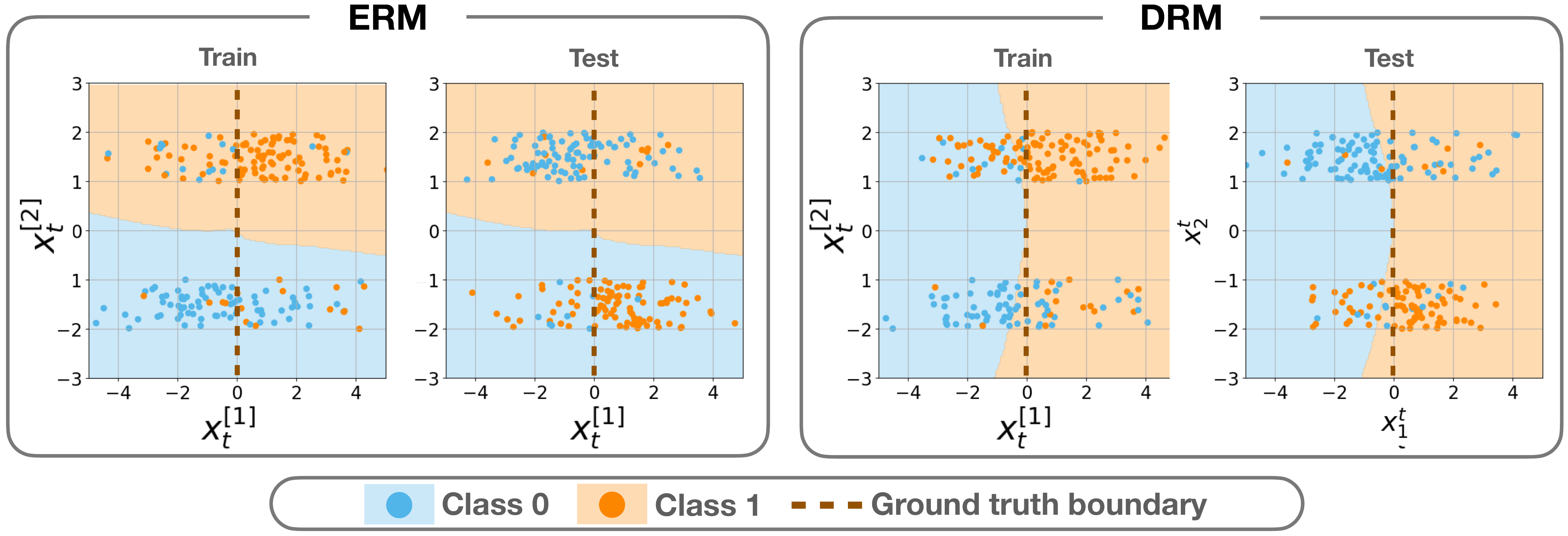}
    \caption{Classifiers learned by ERM and DRM (2D problem). The classification boundary for ERM separates data according to the spurious input dimension $x_t^{[2]}$, which leads to poor performance at test time. DRM disregards $x_t^{[2]}$ almost entirely and learns a classifier that is close to the ground truth ($x_t^{[1]} \geq 0$), leading to strong generalization.}
    \label{fig:toy example representations}
\end{figure}

\subsection{Concept shift: Colored-MNIST}
\label{sec:colored mnist}

Next, we consider the Colored-MNIST task introduced in \cite{arjovsky2019invariant}. The goal is to classify MNIST~\cite{lecun1995learning} images, where the digits have been colored either red or green. Similar to the toy 2D example, the color is assigned in a way that has a strong (but spurious) correlation with the label. As a result, ERM-based methods that only rely on minimizing training loss exploit the color information to make predictions; when the correlation between color and the label is reversed at test time, performance collapses. 

{\bf Training and test distributions.} Concretely, each image is first assigned a preliminary label $\tilde{y} = 0$ for digits $0-4$ and $\tilde{y} = 1$ for digits $5-9$. The final label $y$ is obtained by flipping $\tilde{y}$ with probability 0.25. A color id $c \in \{0, 1\}$ is obtained by flipping $y$ with probability $p_t$, and the image is colored red if $c=1$ and green if $c=0$. The training data sequence consists of examples drawn from two distributions, with the change-point occurring halfway through the data. Specifically, $p_t = 0.1$ for the first half of the data ($t \leq \lceil T/2 \rceil$) and $p_t = 0.4$ for the second half ($t > \lceil T/2 \rceil$). At test time, the probability is chosen to be $p_\text{test} = 0.9$. 

{\bf Results.} We train a convolutional neural network with four layers, and use the (unit-normalized) output of the second layer as our feature $\phi(x_t)$ for computing the DRM martingale penalty (Algorithm~\ref{alg:differentiable drm}). 
Fig.~\ref{fig:all results} (middle) compares the performance of ERM and DRM on the training and test distributions. We also present the performance of IRM, which assumes \emph{oracle knowledge} of the specific point in the training data at which the distribution shift occurs. The reversal of the correlation between the label and the color leads to a significant degradation of performance for ERM. In contrast, DRM achieves a performance that is very similar to IRM, without requiring the training data to be separated into different domains. 

{\bf Visualizing features.} In order to obtain more insight into the representations learned by ERM and DRM, Fig.~\ref{fig:colored-mnist tsne} visualizes the features $\phi^\text{ERM}(x_t)$ and $\phi^\text{DRM}(x_t)$ using their t-SNE embeddings~\cite{maaten2008visualizing}. The embeddings are labeled according to the ground-truth labels (blue: 0, orange: 1) for the corresponding input images, along with the color (red: R or green: G) that was applied to the image. The ERM embeddings form two distinct clusters corresponding to the \emph{color} of the image, confirming that ERM learns to rely almost exclusively on the color rather than the shape of the digit. In contrast, the DRM embeddings are separated based on the label rather than the color. The figure shows a grayscale image colored red or green; these images are mapped to an almost identical t-SNE embedding by DRM, suggesting that DRM has learned to ignore the spurious color information. 

\begin{figure}[t]
    \centering
    \includegraphics[width=0.99\linewidth]{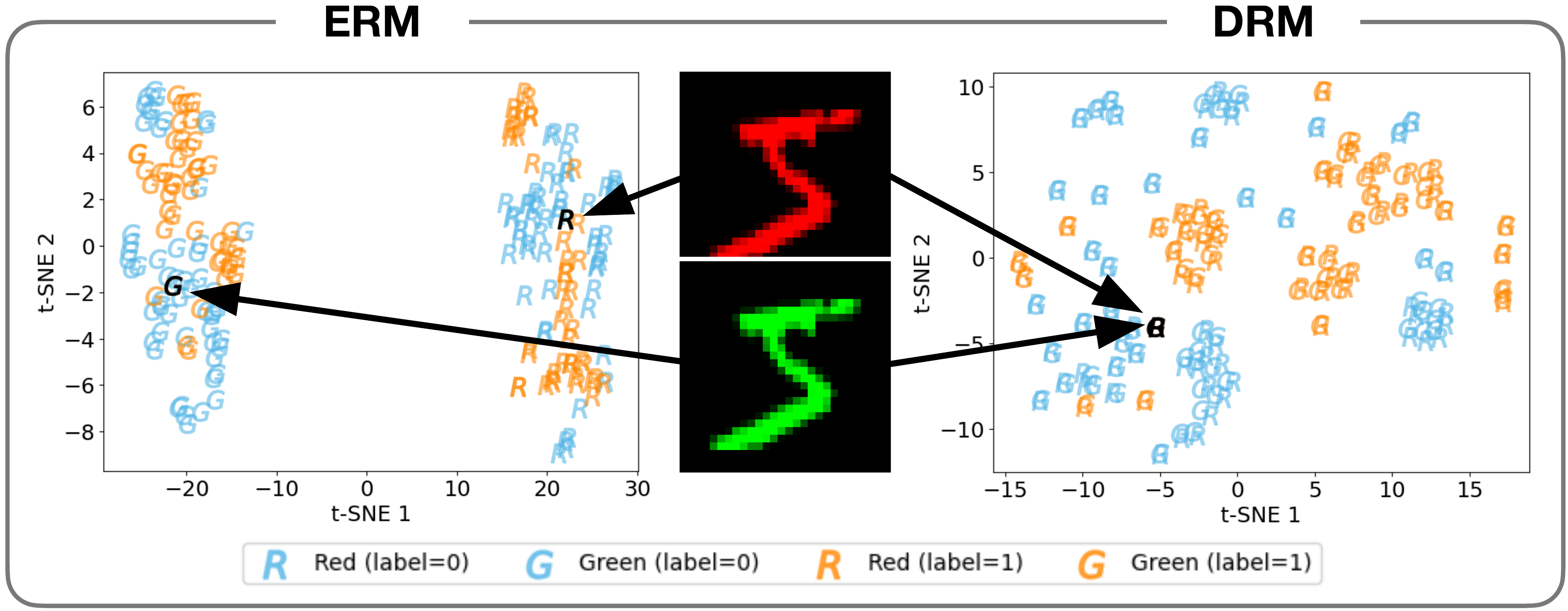}
    \caption{The t-SNE embeddings for features learned by ERM and DRM for Colored-MNIST. ERM embeddings are clustered distinctly by color (R/G). In contrast, DRM embeddings are clustered based on the label, suggesting that DRM has learned to ignore the spurious color information.}
    \label{fig:colored-mnist tsne}
\end{figure}

\subsection{Covariate shift: imitation learning}
\label{sec:imitation learning}


{\bf Training and test distributions.} For our final example, we consider the imitation learning setting from Fig.~\ref{fig:anchor}, which involves covariate shift across environments that the robot is trained and deployed in. The task is to pick up and place a red block into a bowl using observations from an RGB camera. The training data consists of 300 expert demonstrations of pick-and-place locations, which are provided in different environments. A third of the demonstrations are provided with one table-and-bowl color combination (Table RGB: [0, 0.2, 0.7], Bowl RGB: [0, 0, 0.5]), , the next third with a slightly different combination (Table RGB: [0, 0.4, 0.9], Bowl RGB: [0, 0.2, 0.7]), and the final third with another combination (Table RGB: [0, 0.3, 0.6], Bowl RGB: [0, 0.6, 0.3]); these are visualized in Fig.~\ref{fig:anchor}. At test time, the bowl and table background color are changed to a novel combination (Table RGB: [0, 0.9, 0.4], Bowl RGB: [0, 0.7, 0.2]) that significantly exaggerates the variation in green and blue channels seen during training (Fig.~\ref{fig:anchor} right).

{\bf Policy training.} We utilize the transporter network approach~\cite{zeng2021transporter}, which uses two separate neural networks for picking and placing objects. For simplicity, we adopt the same network architecture for both picking and placing the red block (instead of the key-query placing model in \cite{zeng2021transporter}). Each model takes RGB image observations as input. We use residual networks (ResNets)~\cite{He_2016_CVPR} with 36 total layers (convolutional and residual) that form an hourglass encoder-decoder structure. The models are trained to output an image that predicts per-pixel values corresponding to a likelihood that the robot should move to that location for picking / placing. The pick and place models are both trained via a supervised objective in the form of the cross-entropy loss between predicted and demonstrated pick / place locations; demonstration locations are one-hot encoded to form an image that matches the size of the predicted output image. We refer the reader to \cite{zeng2021transporter} for additional details. We find that the picking network is not impacted by the distribution shifts in table and bowl colors (since it is trained to locate the red block, whose color does not change). As a result, we only apply the DRM objective to the placing network. 

{\bf Efficacy of detector.} Fig.~\ref{fig:martingales} visualizes the martingale values on the training data sequence computed using raw image observations, the features learned via ERM, and the features learned via DRM. As the figure illustrates, the conformal martingale-based approach is highly sensitive even to the mild distribution shift that occurs between the first 100 and seconds 100 environments. The martingale value spikes rapidly after the distribution shift for both the raw images and the ERM features. In contrast, the DRM features successfully eliminate the distribution shift from the perspective of the CM. 

{\bf Results.} As shown in Fig.~\ref{fig:all results} (right), deceiving the distribution shift detector results in a policy that is robust to the distribution shift observed between training and testing. In contrast, the near-perfect training performance of a pure behavior cloning objective (ERM) degrades significantly for test environments. 
\section{Related work}
\label{sec:related work}


{\bf Distribution shift detection.} Traditional methods for distribution shift detection use statistical hypothesis testing in order to conclude if a shift has occurred between two sources of data~\cite{gretton2012kernel, rabanser2019failing, kulinski2020feature, farid2024task}. These methods operate in a batch setting, where two datasets (corresponding to training and test distributions) are provided as input for detection. In contrast, DRM relies on \emph{online} methods for distribution shift detection~\cite{vovk2005algorithmic, vovk2003testing, vovk2021testing, ramdas2022testing, shin2022detectors, luo2024online, saha2024testing}, which have been developed relatively recently. These methods are provided with a stream of data, with no demarcation of where a distribution shift may have occurred. Online detectors operate by constructing martingales that test for deviations from the hypothesis of iid (or, more precisely, exchangeable) random variables. In addition to conformal martingales~\cite{vovk2005algorithmic, vovk2003testing, vovk2021testing}, there are also methods based on universal inference~\cite{ramdas2022testing}, e-processes~\cite{shin2022detectors}, and recency prediction~\cite{luo2024online, saha2024testing}. Specialized variants of these approaches have also been developed for the problem of \emph{changepoint detection}~\cite[Ch. 8]{vovk2005algorithmic}, where a distribution shift occurs at a particular (but unknown) time. In addition, theoretical work has characterized the efficiency with which various methods detect distribution shifts~\cite{shin2022detectors, ramdas2022testing}. Our work creates a bridge between the problem of \emph{detecting} distribution shifts and that of \emph{generalizing} to distribution shifts. 

{\bf Domain generalization, invariance, and causality.} Our work is closely related to invariant risk minimization (IRM)~\cite{arjovsky2019invariant}, and the significant amount of subsequent work that it inspired (see, e.g., ~\cite{wang2022generalizing, krueger2021out, ahuja2020invariant, lu2021nonlinear}). IRM and its variants seek to find representations that underlie \emph{causal} mechanisms~\cite{scholkopf2021toward, peters2016causal, peters2017elements} that generate data. Specifically, IRM advocates for a particular kind of \emph{invariance} in order to achieve domain generalization. If we can find a representation of data such that the optimal classifier built on top of this representation is invariant across different domains  (corresponding to different data-generating distributions), this representation will lead to OOD generalization when the underlying causal mechanism is unchanged. This objective is typically approximated via different regularization schemes~\cite{arjovsky2019invariant}, distributionally robust optimization~\cite{krueger2021out}, or via game-theoretic training methods~\cite{ahuja2020invariant}. Practically, the key distinction between IRM and DRM is that we do not assume that data points are associated --- either manually or via unsupervised clustering~\cite{le2025invariant, murata2025clustered} --- with a finite number of data-generating distributions. This assumption is often impractical or not faithful to reality, e.g., in robotics settings where distribution shifts occur \emph{continuously} as data is being collected~\cite{sinha2022system}. Our numerical experiments in Section~\ref{sec:colored mnist} show that DRM can achieve similar performance to IRM without oracular knowledge of distribution shift times. Conceptually, DRM provides a different mechanism for OOD generalization built on the idea of deceiving distribution shift detectors. 

{\bf Domain adaptation and online adaptation.} The objective of aligning training and test distributions also underlies domain adaptation methods, e.g., techniques that align features for training and test distributions~\cite{ben2010theory, ganin2016domain, ganin2015unsupervised, zhang2015multi, long2018conditional, gong2016domain, li2018deep, courty2016optimal}, or ones that re-weight training data points to match the test distribution~\cite{shimodaira2000improving, huang2006correcting, lipton2018detecting}. Domain adaptation methods typically assume that labeled ``source" data points are separately identified from unlabeled or sparsely labeled ``target" data points that come from the test distribution. In contrast, DRM does not assume that data are separated into different sources. Similar to domain generalization methods (e.g., IRM or its variants), we also do not assume access to data from the particular test distribution of interest. Methods that adapt predictors \emph{online} in response to distribution shifts also tackle a related problem to domain adaptation, but focus on settings where the learner is not aware of exact times at which distribution shifts occur~\cite{sun2020test, wang2020tent, zhang2022memo, wang2022continual, li2020online}\cite[Ch. 10]{vovk2005algorithmic}. In contrast to these methods, DRM aims for \emph{zero-shot} OOD generalization to new distributions. 
\section{Discussion and future work}
\label{sec:discussion}

We have introduced \emph{deceptive risk minimization} (DRM): a novel learning objective aimed at identifying stable features that eliminate spurious correlations by hiding distribution shifts from an observer. Our practical instantiation augments a standard ERM loss with a differentiable objective based on conformal martingales. We have provided empirical evidence that DRM can lead to strong generalization to covariate and concept shifts in the presence of spurious correlations in training data. We end with a Q\&A discussion on limitations of DRM, potential ways to address them, and other exciting directions for future work. 

{\bf Q: When does the DRM objective fail to lead to OOD generalization?} \\
{\bf A:} At a conceptual level, DRM seeks to find encodings of data that eliminate distribution shifts from the perspective of an observer. Broadly, there are three failure modes of this objective. First, it may be possible to find representations that make \emph{training} data appear practically iid, but that do not lead to making the combination of training \emph{and} test data practically iid. This can occur if the axes of variation seen in training data do not span differences between training and test data. For example, in the imitation learning example (Sec.~\ref{sec:imitation learning}), the table and bowl colors were only varied along the blue and green channels. As such, a representation that is insensitive to blue-green variations but sensitive to changes in red \emph{would} make training data appear practically iid, but \emph{would not} generalize to test environments where the red channel is changed. A similar but more subtle example is presented in Appendix~\ref{app:failure case}. In order to address such challenges, care should be taken to curate training data in a manner that spans as many relevant axes of variations as possible, even if the \emph{magnitude} of variations is not representative of test data. 

The second failure mode is when the distribution shift detector we are deceiving is not sufficiently powerful. In other words, we may find a representation that deceives the \emph{particular} detector we use, but would not deceive a different detector. We expect that continued progress in distribution shift detection will lead to improvements in DRM. Another particularly promising direction for future work is to simultaneously train both the data representation and the detector as an adversarial game. 

Third, there may be cases where it is not feasible to find representations that eliminate distribution shift in training data, but where one can find invariant prediction rules as advocated by IRM~\cite[Appendix C]{arjovsky2019invariant}. In such cases, DRM is not the right tool to use. We also note that \cite{rosenfeld2020risks} constructs data-generating distributions that cause the IRM objective to fail. Since \cite{rosenfeld2020risks} considers IRM and related objectives that find invariances across a finite number of data-generating distributions, the results are not directly applicable to DRM. An interesting theoretical direction is to characterize the precise conditions under which a DRM-style objective can lead to OOD generalization. We provide a preliminary sketch of theoretical underpinnings of DRM in Appendix~\ref{app:theory} by connecting deception to generalization.

{\bf Q: What are the computational challenges related to implementing DRM?} \\
{\bf A:} The primary computational bottleneck is in Eq.~\eqref{eq:conformity score covariate} and Eq.~\eqref{eq:conformity score concept}, which compute the conformity scores for each example. For each example in the sequence of data points used for distribution shift detection, we need to compute the minimum distance in embedding space to other examples in the sequence. The most straightforward implementation of this (which we adopt in our numerical experiments) incurs a complexity that grows quadratically with the length of the sequence used for detection. Currently, we address this issue by sampling subsequences of data from the (potentially large) training sequence, and using these to compute martingale values which are then averaged (Sec.~\ref{sec:implementation details}). Finding strategies to improve this computational bottleneck --- perhaps with inspiration from efficient implementations of the quadratic-complexity attention mechanism~\cite{zhuang2023survey} --- is an important avenue for making DRM scalable. 

{\bf Q: How sensitive is DRM to different hyperparameters?} \\
{\bf A:} The primary hyperparameters in DRM are: the dispersion parameter $\sigma$ for smooth sorting (Sec.~\ref{sec:differentiable objective}), the length of the sequences used for distribution shift detection (Sec.~\ref{sec:implementation details}), and the relative weighting $\lambda$ between the ERM objective and the DRM regularization (Eq.~\eqref{eq:drm objective differentiable}). Hyperparameters chosen for the numerical experiments are reported in Appendix~\ref{app:hyperparameters}, and we present results from a hyperparameter sweep for the 2D example in Appendix~\ref{app:hyperparameter sweep}. We find that DRM is sensitive to the dispersion parameter $\sigma$ and relatively insensitive to the weighting $\lambda$ and the length of the detection sequences.

{\bf Q: Can other methods be used for distribution shift detection in place of conformal test martingales?} \\
{\bf A:} In this work, we instantiated DRM using conformal martingales (CMs). This choice was motivated by (i) prior work that demonstrates the ability of CMs to detect distribution shifts rapidly~\cite{vovk2005algorithmic}, (ii) the ability of CMs to detect different kinds of distribution shifts (e.g., covariate and concept shifts), and (iii) the fact that we can construct a differentiable surrogate for CMs. There is exciting future work in contrasting the theoretical and empirical benefits of utilizing other approaches to distribution shift detection (Sec.~\ref{sec:related work}). An approach based on a different detector may alleviate some of the computational challenges highlighted above. 

{\bf Q: Are there other kinds of distribution shift that could be handled by a DRM-style objective?} \\
{\bf A:} The two kinds of distribution shift we have considered in this paper are covariate shift and (anti-causal) concept shift. Chapter 8.2 of \cite{vovk2005algorithmic} presents a conformal martingale for detecting \emph{label shift}, i.e., a shift in the marginal distribution of class labels. 
Causal concept shift--- a change in the conditional distribution $Y|X$ --- is highly relevant in causal inference~\cite{peters2017elements}. In the absence of label shift and covariate shift, anti-causal and causal concept shift are equivalent (via Bayes' rule), as in our examples from Sec.~\ref{sec:toy example} and \ref{sec:colored mnist}. Extending DRM to tackle causal concept shift in general settings is an important avenue for future work. 

{\bf Q: Could DRM be used for covariate shifts due to compounding errors in imitation learning?} \\
{\bf A:} One idea is to implement the iterative data collection process in DAGGER (dataset aggregation)~\cite{ross2011reduction}, and use DRM to find features that remain robust to the covariate shift between the states visited in successive iterations. Such a strategy may lead to more robust policies compared to DAGGER, which re-trains the policy by aggregating data across iterations of data collection. Working out the details of such an approach could make for interesting future work. 

{\bf Q: Could DRM be used for reinforcement learning?} \\
{\bf A:} One immediate application of DRM in reinforcement learning (RL) is in the setting where one has access to a sequence of Markov decision processes (MDPs) for training (similar to the imitation learning setup considered in Sec.~\ref{sec:imitation learning}). In this case, the distribution shift detector can take as input observations from different environments, and the DRM objective would then attempt to learn a policy whose features appear iid across environments. 

\quad

Overall, we are excited by the prospect that the bridge between distribution shift detection and generalization provided by DRM will lead to new techniques that address the problem of OOD generalization, which remains prevalent despite the scale of modern machine learning.  

\section*{Acknowledgements}
This work was supported by the Office of Naval Research (N00014-23-1-2148). The author is grateful to May Mei, Ola Shorinwa, Asher Hancock, David Snyder, and Apurva Badithela for helpful feedback on the paper. 


\bibliographystyle{unsrtnat}
\bibliography{references.bib}

\begin{thebibliography}{57}
\providecommand{\natexlab}[1]{#1}
\providecommand{\url}[1]{\texttt{#1}}
\expandafter\ifx\csname urlstyle\endcsname\relax
  \providecommand{\doi}[1]{doi: #1}\else
  \providecommand{\doi}{doi: \begingroup \urlstyle{rm}\Url}\fi

\bibitem[Shalev-Shwartz and Ben-David(2014)]{shalev2014understanding}
Shai Shalev-Shwartz and Shai Ben-David.
\newblock \emph{Understanding Machine Learning: From Theory to Algorithms}.
\newblock Cambridge University Press, 2014.

\bibitem[Liu et~al.(2021)Liu, Shen, He, Zhang, Xu, Yu, and Cui]{liu2021towards}
Jiashuo Liu, Zheyan Shen, Yue He, Xingxuan Zhang, Renzhe Xu, Han Yu, and Peng
  Cui.
\newblock Towards out-of-distribution generalization: A survey.
\newblock \emph{arXiv preprint arXiv:2108.13624}, 2021.

\bibitem[Sinha et~al.(2022)Sinha, Sharma, Banerjee, Lew, Luo, Richards, Sun,
  Schmerling, and Pavone]{sinha2022system}
Rohan Sinha, Apoorva Sharma, Somrita Banerjee, Thomas Lew, Rachel Luo,
  Spencer~M Richards, Yixiao Sun, Edward Schmerling, and Marco Pavone.
\newblock A system-level view on out-of-distribution data in robotics.
\newblock \emph{arXiv preprint arXiv:2212.14020}, 2022.

\bibitem[Li et~al.(2025)Li, Rubungo, Lei, Persaud, Choudhary, DeCost, Dieng,
  and Hattrick-Simpers]{li2025probing}
Kangming Li, Andre~Niyongabo Rubungo, Xiangyun Lei, Daniel Persaud, Kamal
  Choudhary, Brian DeCost, Adji~Bousso Dieng, and Jason Hattrick-Simpers.
\newblock Probing out-of-distribution generalization in machine learning for
  materials.
\newblock \emph{Communications Materials}, 6\penalty0 (1):\penalty0 9, 2025.

\bibitem[Arjovsky et~al.(2019)Arjovsky, Bottou, Gulrajani, and
  Lopez-Paz]{arjovsky2019invariant}
Martin Arjovsky, L{\'e}on Bottou, Ishaan Gulrajani, and David Lopez-Paz.
\newblock Invariant risk minimization.
\newblock \emph{arXiv preprint arXiv:1907.02893}, 2019.

\bibitem[Ben-David et~al.(2010)Ben-David, Blitzer, Crammer, Kulesza, Pereira,
  and Vaughan]{ben2010theory}
Shai Ben-David, John Blitzer, Koby Crammer, Alex Kulesza, Fernando Pereira, and
  Jennifer~Wortman Vaughan.
\newblock A theory of learning from different domains.
\newblock \emph{Machine Learning}, 79\penalty0 (1):\penalty0 151--175, 2010.

\bibitem[Ganin et~al.(2016)Ganin, Ustinova, Ajakan, Germain, Larochelle,
  Laviolette, March, and Lempitsky]{ganin2016domain}
Yaroslav Ganin, Evgeniya Ustinova, Hana Ajakan, Pascal Germain, Hugo
  Larochelle, Fran{\c{c}}ois Laviolette, Mario March, and Victor Lempitsky.
\newblock Domain-adversarial training of neural networks.
\newblock \emph{Journal of Machine Learning Research}, 17\penalty0
  (59):\penalty0 1--35, 2016.

\bibitem[Ganin and Lempitsky(2015)]{ganin2015unsupervised}
Yaroslav Ganin and Victor Lempitsky.
\newblock Unsupervised domain adaptation by backpropagation.
\newblock In \emph{International Conference on Machine Learning}, pages
  1180--1189. PMLR, 2015.

\bibitem[Zhang et~al.(2015)Zhang, Gong, and Sch{\"o}lkopf]{zhang2015multi}
Kun Zhang, Mingming Gong, and Bernhard Sch{\"o}lkopf.
\newblock Multi-source domain adaptation: A causal view.
\newblock In \emph{Proceedings of the AAAI Conference on Artificial
  Intelligence}, volume~29, 2015.

\bibitem[Long et~al.(2018)Long, Cao, Wang, and Jordan]{long2018conditional}
Mingsheng Long, Zhangjie Cao, Jianmin Wang, and Michael~I Jordan.
\newblock Conditional adversarial domain adaptation.
\newblock \emph{Advances in Neural Information Processing Systems}, 31, 2018.

\bibitem[Gong et~al.(2016)Gong, Zhang, Liu, Tao, Glymour, and
  Sch{\"o}lkopf]{gong2016domain}
Mingming Gong, Kun Zhang, Tongliang Liu, Dacheng Tao, Clark Glymour, and
  Bernhard Sch{\"o}lkopf.
\newblock Domain adaptation with conditional transferable components.
\newblock In \emph{International Conference on Machine Learning}, pages
  2839--2848. PMLR, 2016.

\bibitem[Li et~al.(2018)Li, Tian, Gong, Liu, Liu, Zhang, and Tao]{li2018deep}
Ya~Li, Xinmei Tian, Mingming Gong, Yajing Liu, Tongliang Liu, Kun Zhang, and
  Dacheng Tao.
\newblock Deep domain generalization via conditional invariant adversarial
  networks.
\newblock In \emph{Proceedings of the European Conference on Computer Vision
  (ECCV)}, pages 624--639, 2018.

\bibitem[Courty et~al.(2016)Courty, Flamary, Tuia, and
  Rakotomamonjy]{courty2016optimal}
Nicolas Courty, R{\'e}mi Flamary, Devis Tuia, and Alain Rakotomamonjy.
\newblock Optimal transport for domain adaptation.
\newblock \emph{IEEE Transactions on Pattern Analysis and Machine
  Intelligence}, 39\penalty0 (9):\penalty0 1853--1865, 2016.

\bibitem[Peters et~al.(2016)Peters, B{\"u}hlmann, and
  Meinshausen]{peters2016causal}
Jonas Peters, Peter B{\"u}hlmann, and Nicolai Meinshausen.
\newblock Causal inference by using invariant prediction: identification and
  confidence intervals.
\newblock \emph{Journal of the Royal Statistical Society Series B: Statistical
  Methodology}, 78\penalty0 (5):\penalty0 947--1012, 2016.

\bibitem[Ahuja et~al.(2020)Ahuja, Shanmugam, Varshney, and
  Dhurandhar]{ahuja2020invariant}
Kartik Ahuja, Karthikeyan Shanmugam, Kush Varshney, and Amit Dhurandhar.
\newblock Invariant risk minimization games.
\newblock In \emph{International Conference on Machine Learning}, pages
  145--155. PMLR, 2020.

\bibitem[Krueger et~al.(2021)Krueger, Caballero, Jacobsen, Zhang, Binas, Zhang,
  Le~Priol, and Courville]{krueger2021out}
David Krueger, Ethan Caballero, Joern-Henrik Jacobsen, Amy Zhang, Jonathan
  Binas, Dinghuai Zhang, Remi Le~Priol, and Aaron Courville.
\newblock Out-of-distribution generalization via risk extrapolation ({RE}x).
\newblock In \emph{International Conference on Machine Learning}, pages
  5815--5826. PMLR, 2021.

\bibitem[Le et~al.(2025)Le, Seifert, and Schl{\"o}tterer]{le2025invariant}
Phuong~Quynh Le, Christin Seifert, and J{\"o}rg Schl{\"o}tterer.
\newblock Invariant learning with annotation-free environments.
\newblock \emph{arXiv preprint arXiv:2504.15686}, 2025.

\bibitem[Murata et~al.(2025)Murata, Nitanda, and Suzuki]{murata2025clustered}
Tomoya Murata, Atsushi Nitanda, and Taiji Suzuki.
\newblock Clustered invariant risk minimization.
\newblock In \emph{The 28th International Conference on Artificial Intelligence
  and Statistics}, 2025.

\bibitem[Vovk et~al.(2022)Vovk, Gammerman, and Shafer]{vovk2005algorithmic}
Vladimir Vovk, Alexander Gammerman, and Glenn Shafer.
\newblock \emph{Algorithmic Learning in a Random World}.
\newblock Springer, 2022.

\bibitem[Vovk et~al.(2003)Vovk, Nouretdinov, and Gammerman]{vovk2003testing}
Vladimir Vovk, Ilia Nouretdinov, and Alexander Gammerman.
\newblock Testing exchangeability on-line.
\newblock In \emph{Proceedings of the International Conference on Machine
  Learning}, pages 768--775, 2003.

\bibitem[Vovk(2021)]{vovk2021testing}
Vladimir Vovk.
\newblock Testing randomness online.
\newblock \emph{Statistical Science}, 36\penalty0 (4):\penalty0 595--611, 2021.

\bibitem[Gillies(2012)]{gillies2012philosophical}
Donald Gillies.
\newblock \emph{Philosophical Theories of Probability}.
\newblock Routledge, 2012.

\bibitem[De~Finetti(2017)]{de2017theory}
Bruno De~Finetti.
\newblock \emph{Theory of Probability: A Critical Introductory Treatment}.
\newblock John Wiley \& Sons, 2017.

\bibitem[Vadhan et~al.(2012)]{vadhan2012pseudorandomness}
Salil~P Vadhan et~al.
\newblock Pseudorandomness.
\newblock \emph{Foundations and Trends{\textregistered} in Theoretical Computer
  Science}, 7\penalty0 (1--3):\penalty0 1--336, 2012.

\bibitem[Ramdas et~al.(2022)Ramdas, Ruf, Larsson, and
  Koolen]{ramdas2022testing}
Aaditya Ramdas, Johannes Ruf, Martin Larsson, and Wouter~M Koolen.
\newblock Testing exchangeability: Fork-convexity, supermartingales and
  e-processes.
\newblock \emph{International Journal of Approximate Reasoning}, 141:\penalty0
  83--109, 2022.

\bibitem[Team et~al.(2025)Team, Abeyruwan, Ainslie, Alayrac, Arenas, Armstrong,
  Balakrishna, Baruch, Bauza, Blokzijl, et~al.]{team2025gemini}
Gemini~Robotics Team, Saminda Abeyruwan, Joshua Ainslie, Jean-Baptiste Alayrac,
  Montserrat~Gonzalez Arenas, Travis Armstrong, Ashwin Balakrishna, Robert
  Baruch, Maria Bauza, Michiel Blokzijl, et~al.
\newblock Gemini robotics: Bringing {AI} into the physical world.
\newblock \emph{arXiv preprint arXiv:2503.20020}, 2025.

\bibitem[Intelligence et~al.(2025)Intelligence, Black, Brown, Darpinian,
  Dhabalia, Driess, Esmail, Equi, Finn, Fusai, et~al.]{intelligence2025pi}
Physical Intelligence, Kevin Black, Noah Brown, James Darpinian, Karan
  Dhabalia, Danny Driess, Adnan Esmail, Michael Equi, Chelsea Finn, Niccolo
  Fusai, et~al.
\newblock $\pi_{0.5}$: a vision-language-action model with open-world
  generalization.
\newblock \emph{arXiv preprint arXiv:2504.16054}, 2025.

\bibitem[Ahmad and Mazzara(2024)]{ahmad2024scsnet}
Muhammad Ahmad and Manuel Mazzara.
\newblock Scsnet: Sharpened cosine similarity-based neural network for
  hyperspectral image classification.
\newblock \emph{IEEE Geoscience and Remote Sensing Letters}, 21:\penalty0 1--4,
  2024.

\bibitem[Stutz et~al.(2021)Stutz, Cemgil, Doucet, et~al.]{stutz2021learning}
David Stutz, Ali~Taylan Cemgil, Arnaud Doucet, et~al.
\newblock Learning optimal conformal classifiers.
\newblock \emph{arXiv preprint arXiv:2110.09192}, 2021.

\bibitem[Blondel et~al.(2020)Blondel, Teboul, Berthet, and
  Djolonga]{blondel2020fast}
Mathieu Blondel, Olivier Teboul, Quentin Berthet, and Josip Djolonga.
\newblock Fast differentiable sorting and ranking.
\newblock In \emph{International Conference on Machine Learning}, pages
  950--959. PMLR, 2020.

\bibitem[Cuturi et~al.(2019)Cuturi, Teboul, and Vert]{cuturi2019differentiable}
Marco Cuturi, Olivier Teboul, and Jean-Philippe Vert.
\newblock Differentiable ranking and sorting using optimal transport.
\newblock \emph{Advances in Neural Information Processing Systems}, 32, 2019.

\bibitem[Wang et~al.(2022{\natexlab{a}})Wang, Lan, Liu, Ouyang, Qin, Lu, Chen,
  Zeng, and Yu]{wang2022generalizing}
Jindong Wang, Cuiling Lan, Chang Liu, Yidong Ouyang, Tao Qin, Wang Lu, Yiqiang
  Chen, Wenjun Zeng, and Philip~S Yu.
\newblock Generalizing to unseen domains: A survey on domain generalization.
\newblock \emph{IEEE Transactions on Knowledge and Data Engineering},
  35\penalty0 (8):\penalty0 8052--8072, 2022{\natexlab{a}}.

\bibitem[Lu et~al.(2021)Lu, Wu, Hern{\'a}ndez-Lobato, and
  Sch{\"o}lkopf]{lu2021nonlinear}
Chaochao Lu, Yuhuai Wu, Jo{\'s}e~Miguel Hern{\'a}ndez-Lobato, and Bernhard
  Sch{\"o}lkopf.
\newblock Nonlinear invariant risk minimization: A causal approach.
\newblock \emph{arXiv preprint arXiv:2102.12353}, 2021.

\bibitem[LeCun et~al.(1995)LeCun, Jackel, Bottou, Cortes, Denker, Drucker,
  Guyon, Muller, Sackinger, Simard, et~al.]{lecun1995learning}
Yann LeCun, Lawrence~D Jackel, L{\'e}on Bottou, Corinna Cortes, John~S Denker,
  Harris Drucker, Isabelle Guyon, Urs~A Muller, Eduard Sackinger, Patrice
  Simard, et~al.
\newblock Learning algorithms for classification: A comparison on handwritten
  digit recognition.
\newblock \emph{Neural networks: the statistical mechanics perspective},
  261\penalty0 (276):\penalty0 2, 1995.

\bibitem[Maaten and Hinton(2008)]{maaten2008visualizing}
Laurens van~der Maaten and Geoffrey Hinton.
\newblock Visualizing data using t-sne.
\newblock \emph{Journal of Machine Learning Research}, 9\penalty0
  (Nov):\penalty0 2579--2605, 2008.

\bibitem[Zeng et~al.(2021)Zeng, Florence, Tompson, Welker, Chien, Attarian,
  Armstrong, Krasin, Duong, Sindhwani, et~al.]{zeng2021transporter}
Andy Zeng, Pete Florence, Jonathan Tompson, Stefan Welker, Jonathan Chien,
  Maria Attarian, Travis Armstrong, Ivan Krasin, Dan Duong, Vikas Sindhwani,
  et~al.
\newblock Transporter networks: Rearranging the visual world for robotic
  manipulation.
\newblock In \emph{Conference on Robot Learning}, pages 726--747. PMLR, 2021.

\bibitem[He et~al.(2016)He, Zhang, Ren, and Sun]{He_2016_CVPR}
Kaiming He, Xiangyu Zhang, Shaoqing Ren, and Jian Sun.
\newblock Deep residual learning for image recognition.
\newblock In \emph{Proceedings of the IEEE Conference on Computer Vision and
  Pattern Recognition (CVPR)}, June 2016.

\bibitem[Gretton et~al.(2012)Gretton, Borgwardt, Rasch, Sch{\"o}lkopf, and
  Smola]{gretton2012kernel}
Arthur Gretton, Karsten~M Borgwardt, Malte~J Rasch, Bernhard Sch{\"o}lkopf, and
  Alexander Smola.
\newblock A kernel two-sample test.
\newblock \emph{The Journal of Machine Learning Research}, 13\penalty0
  (1):\penalty0 723--773, 2012.

\bibitem[Rabanser et~al.(2019)Rabanser, G{\"u}nnemann, and
  Lipton]{rabanser2019failing}
Stephan Rabanser, Stephan G{\"u}nnemann, and Zachary Lipton.
\newblock Failing loudly: An empirical study of methods for detecting dataset
  shift.
\newblock \emph{Advances in Neural Information Processing Systems}, 32, 2019.

\bibitem[Kulinski et~al.(2020)Kulinski, Bagchi, and
  Inouye]{kulinski2020feature}
Sean Kulinski, Saurabh Bagchi, and David~I Inouye.
\newblock Feature shift detection: Localizing which features have shifted via
  conditional distribution tests.
\newblock \emph{Advances in Neural Information Processing Systems},
  33:\penalty0 19523--19533, 2020.

\bibitem[Farid et~al.(2024)Farid, Veer, Pachisia, and Majumdar]{farid2024task}
Alec Farid, Sushant Veer, Divyanshu Pachisia, and Anirudha Majumdar.
\newblock Task-driven detection of distribution shifts with statistical
  guarantees for robot learning.
\newblock \emph{IEEE Transactions on Robotics}, 2024.

\bibitem[Shin et~al.(2022)Shin, Ramdas, and Rinaldo]{shin2022detectors}
Jaehyeok Shin, Aaditya Ramdas, and Alessandro Rinaldo.
\newblock E-detectors: A nonparametric framework for sequential change
  detection.
\newblock \emph{arXiv preprint arXiv:2203.03532}, 2022.

\bibitem[Luo et~al.(2024)Luo, Sinha, Sun, Hindy, Zhao, Savarese, Schmerling,
  and Pavone]{luo2024online}
Rachel Luo, Rohan Sinha, Yixiao Sun, Ali Hindy, Shengjia Zhao, Silvio Savarese,
  Edward Schmerling, and Marco Pavone.
\newblock Online distribution shift detection via recency prediction.
\newblock In \emph{2024 IEEE International Conference on Robotics and
  Automation (ICRA)}, pages 16251--16263. IEEE, 2024.

\bibitem[Saha and Ramdas(2024)]{saha2024testing}
Aytijhya Saha and Aaditya Ramdas.
\newblock Testing exchangeability by pairwise betting.
\newblock In \emph{International Conference on Artificial Intelligence and
  Statistics}, pages 4915--4923. PMLR, 2024.

\bibitem[Sch{\"o}lkopf et~al.(2021)Sch{\"o}lkopf, Locatello, Bauer, Ke,
  Kalchbrenner, Goyal, and Bengio]{scholkopf2021toward}
Bernhard Sch{\"o}lkopf, Francesco Locatello, Stefan Bauer, Nan~Rosemary Ke, Nal
  Kalchbrenner, Anirudh Goyal, and Yoshua Bengio.
\newblock Toward causal representation learning.
\newblock \emph{Proceedings of the IEEE}, 109\penalty0 (5):\penalty0 612--634,
  2021.

\bibitem[Peters et~al.(2017)Peters, Janzing, and
  Sch{\"o}lkopf]{peters2017elements}
Jonas Peters, Dominik Janzing, and Bernhard Sch{\"o}lkopf.
\newblock \emph{Elements of Causal Inference: Foundations and Learning
  Algorithms}.
\newblock The MIT press, 2017.

\bibitem[Shimodaira(2000)]{shimodaira2000improving}
Hidetoshi Shimodaira.
\newblock Improving predictive inference under covariate shift by weighting the
  log-likelihood function.
\newblock \emph{Journal of Statistical Planning and Inference}, 90\penalty0
  (2):\penalty0 227--244, 2000.

\bibitem[Huang et~al.(2006)Huang, Gretton, Borgwardt, Sch{\"o}lkopf, and
  Smola]{huang2006correcting}
Jiayuan Huang, Arthur Gretton, Karsten Borgwardt, Bernhard Sch{\"o}lkopf, and
  Alex Smola.
\newblock Correcting sample selection bias by unlabeled data.
\newblock \emph{Advances in Neural Information Processing Systems}, 19, 2006.

\bibitem[Lipton et~al.(2018)Lipton, Wang, and Smola]{lipton2018detecting}
Zachary Lipton, Yu-Xiang Wang, and Alexander Smola.
\newblock Detecting and correcting for label shift with black box predictors.
\newblock In \emph{International Conference on Machine Learning}, pages
  3122--3130. PMLR, 2018.

\bibitem[Sun et~al.(2020)Sun, Wang, Liu, Miller, Efros, and Hardt]{sun2020test}
Yu~Sun, Xiaolong Wang, Zhuang Liu, John Miller, Alexei Efros, and Moritz Hardt.
\newblock Test-time training with self-supervision for generalization under
  distribution shifts.
\newblock In \emph{International Conference on Machine Learning}, pages
  9229--9248. PMLR, 2020.

\bibitem[Wang et~al.(2020)Wang, Shelhamer, Liu, Olshausen, and
  Darrell]{wang2020tent}
Dequan Wang, Evan Shelhamer, Shaoteng Liu, Bruno Olshausen, and Trevor Darrell.
\newblock Tent: Fully test-time adaptation by entropy minimization.
\newblock \emph{arXiv preprint arXiv:2006.10726}, 2020.

\bibitem[Zhang et~al.(2022)Zhang, Levine, and Finn]{zhang2022memo}
Marvin Zhang, Sergey Levine, and Chelsea Finn.
\newblock Memo: Test time robustness via adaptation and augmentation.
\newblock \emph{Advances in Neural Information Processing Systems},
  35:\penalty0 38629--38642, 2022.

\bibitem[Wang et~al.(2022{\natexlab{b}})Wang, Fink, Van~Gool, and
  Dai]{wang2022continual}
Qin Wang, Olga Fink, Luc Van~Gool, and Dengxin Dai.
\newblock Continual test-time domain adaptation.
\newblock In \emph{Proceedings of the IEEE/CVF Conference on Computer Vision
  and Pattern Recognition}, pages 7201--7211, 2022{\natexlab{b}}.

\bibitem[Li and Hospedales(2020)]{li2020online}
Da~Li and Timothy Hospedales.
\newblock Online meta-learning for multi-source and semi-supervised domain
  adaptation.
\newblock In \emph{European Conference on Computer Vision}, pages 382--403.
  Springer, 2020.

\bibitem[Rosenfeld et~al.(2020)Rosenfeld, Ravikumar, and
  Risteski]{rosenfeld2020risks}
Elan Rosenfeld, Pradeep Ravikumar, and Andrej Risteski.
\newblock The risks of invariant risk minimization.
\newblock \emph{arXiv preprint arXiv:2010.05761}, 2020.

\bibitem[Zhuang et~al.(2023)Zhuang, Liu, Pan, He, Weng, and
  Shen]{zhuang2023survey}
Bohan Zhuang, Jing Liu, Zizheng Pan, Haoyu He, Yuetian Weng, and Chunhua Shen.
\newblock A survey on efficient training of transformers.
\newblock \emph{arXiv preprint arXiv:2302.01107}, 2023.

\bibitem[Ross et~al.(2011)Ross, Gordon, and Bagnell]{ross2011reduction}
St{\'e}phane Ross, Geoffrey Gordon, and Drew Bagnell.
\newblock A reduction of imitation learning and structured prediction to
  no-regret online learning.
\newblock In \emph{Proceedings of the International Conference on Artificial
  Intelligence and Statistics}, pages 627--635. JMLR Workshop and Conference
  Proceedings, 2011.

\end{thebibliography}


\clearpage
\beginappendix{
    \section{Hyperparameters for numerical experiments}
\label{app:hyperparameters}

\begin{table}[h!]
\centering
\renewcommand{\arraystretch}{1.3} 
\begin{tabular}{l|ccc}
\hline
\textbf{Parameter} & \textbf{Toy 2D Example} & \textbf{Colored-MNIST} & \textbf{Imitation Learning} \\
\hline
\# training examples (T) & 2000 & 2000 & 300 \\
ERM loss batch size          & 64 & 64 & 64 \\
Size of detection sequence   & 1000 & 1000 & 200 \\
\# of detection sequences    & 1 & 3 & 3 \\
Regularization weight ($\lambda$) & 5e5 & 5e6 & 1e4 \\
Dispersion for soft-ranking ($\sigma$) & 0.001 & 0.1 & 0.001 \\
Learning rate                 & 0.005 & 0.005 & 0.001 \\
\# ERM epochs                 & 0 & 2 & 0 \\
\# total training epochs      & 2 & 3 & 25 \\
\hline
\end{tabular}
\end{table}

\section{Hyperparameter sweep for 2D example}
\label{app:hyperparameter sweep}

The following table shows success rates for different values of the regularization weight ($\lambda$), averaged across 5 training seeds. 

\begin{table}[H]
\centering
\begin{tabular}{|c|c|c|c|c|}
\hline
 & $\lambda =$ 1e4 & $\lambda =$ 1e5 & $\lambda =$ 5e5 & $\lambda =$ 1e6 \\
\hline
Success (Train $\|$ Test) & 0.67 $\|$ 0.72 & 0.74 $\|$ 0.57 & 0.70 $\|$ 0.63 & 0.69 $\|$ 0.56 \\
\hline
\end{tabular}
\end{table}

The following table shows success rates for different values of the detection sequence length, averaged across 5 training seeds. 

\begin{table}[H]
\centering
\begin{tabular}{|c|c|c|c|c|c|}
\hline
 & 200 & 400 & 600 & 800 & 1000 \\
\hline
Success (Train $\|$ Test) & 0.69 $\|$ 0.53  & 0.69 $\|$ 0.68 & 0.71 $\|$ 0.63 & 0.71 $\|$ 0.56 & 0.70 $\|$ 0.63 \\
\hline
\end{tabular}
\end{table}

The following table shows success rates for different values of the soft-rank dispersion parameter ($\sigma$), averaged across 5 training seeds. 

\begin{table}[H]
\centering
\begin{tabular}{|c|c|c|c|}
\hline
 & $\sigma =$ 0.001 & $\sigma =$ 0.01 & $\sigma =$ 0.1 \\
\hline
Success (Train $\|$ Test) & 0.70 $\|$ 0.63  & 0.65 $\|$ 0.56 & 0.71 $\|$ 0.51  \\
\hline
\end{tabular}
\end{table}

\section{Example where DRM objective is insufficient}
\label{app:failure case}

We discuss an example where the DRM objective is not sufficient for generalization to test data. We consider a version of the imitation learning example from Section~\ref{sec:imitation learning}, where the training data consists of two portions: the first half has bowls with RGB: [0, 0.1, 0.5] and table with RGB: [0, 0.5, 0.1], and the second half has bowls with RGB: [0, 0.2, 0.4] and table with RGB: [0, 0.4, 0.2]. Here, there are multiple ways in which to make the training data appear iid. A representation that is insensitive to blue-green variations would suffice, and would also generalize to unseen blue-green settings. However, a representation that detects a bowl if G (green) $<$ 0.3 and B (blue) $>$ 0.3 would also make training data appear iid, but would not generalize to a test environment where the bowl has RGB: [0, 0.9, 0.1] and the table has RGB: [0, 0.1, 0.9]. 

\section{Deceive to generalize: theoretical intuitions}
\label{app:theory}

In this section, we draw theoretical connections between the objective of deceiving a distribution shift detector and that of achieving OOD generalization. The connection is made in three parts. First, we demonstrate that if a particular distribution shift detector $\Delta^\star$ can be deceived into concluding that the random variables corresponding to training and test losses are iid, then the expected test loss is very close to the expected training loss. Second, we allow for detectors that take encoded representations $\phi(x)$ as input instead of loss values. Third, we define the $\Delta$-span of a representation learned from training random variables as containing test distributions such that training and test random variables are practically iid. Any test distribution in the span thus has expected test loss close to the expected training loss. 

\subsection{Efficiency of distribution shift detection}

In Sec.~\ref{sec:observer}, we defined observers $\Delta$ in the form of distribution shift detectors that control the false alarm rate (FAR). A detector should also ideally detect distribution shifts as quickly as possible. The notion of efficiency can be formalized by the worse average delay (WAD) of a detector. 

{\bf Worst average delay (WAD).} Suppose that the marginal distributions of the sequence of random variables $(\phi(X_1), \phi(X_2), \dots)$ change at an unknown time $\nu$, referred to as a \emph{changepoint}. The worst average delay (WAD) in detecting the change is~\cite{shin2022detectors}:
\begin{equation}
    \sup_{\nu \geq 0} \ \mathbb{E} [N^\star - \nu | N^\star > \nu],
\end{equation}
where $N^\star$ is the time at which a distribution shift is declared ($N^\star = \infty$ if a change is never declared). 

The following definition formalizes the idea of random variables with a changepoint appearing iid to a given observer. Intuitively, the sequence of random variables is practically iid if the worst average delay in detecting a changepoint is large. 

\begin{definition}[Practically iid w/ changepoint]
A sequence of random variables $(\phi(X_1), \dots, \phi(X_\nu), \phi(X_{\nu+1}), \dots)$ with changepoint $\nu$ is $(\Delta_{\alpha}, \epsilon)$-\emph{practically iid} if the detector $\Delta_{\alpha}$ with FAR bounded by $\alpha$ has a large WAD in detecting the changepoint: WAD $> (1/\epsilon) \log(1/\alpha)$. 
\end{definition}

\subsection{Connecting detection to generalization}

Consider the sequence of input random variables $(X_1, \dots, X_T, X_{T+1} \dots)$ as in Sec.~\ref{sec:problem formulation}, where the changepoint $T$ separates training and test distributions. We will demonstrate that there is an encoding $\phi$ of inputs and a detector $\Delta^\star_\alpha$ such that if $(\phi(X_1), \dots, \phi(X_T), \phi(X_{T+1}), \dots)$ is $(\Delta_\alpha^\star, \epsilon)$-practically iid, then the expected test loss is close to the expected training loss. 

\begin{proposition}
\label{prop:generalization}
Let $h$ be a hypothesis that maps inputs to labels, and consider a binary-valued loss function, i.e., $l(x, h(x)) \in \{0,1\}, \forall x$. Suppose that the expected loss under the training random variables is bounded as follows:
\begin{equation}
    \mathbb{E} [l(X_t, h(X_t)) | \mathcal{F}_{t-1})] \leq l_\text{train}, \ \forall t \leq T,
\end{equation}
where $\mathcal{F}_{t-1}$ denotes the natural filtration of the data. 
Consider the sequence of random variables $(X_1, \dots, X_T, \dots)$, where the test random variables $(X_{T+1}, X_{T+2}, \dots)$ are iid. Then the expected test loss is:
\begin{equation}
    {l_\text{test}} := \mathbb{E} [l(X_{T+1}, h(X_{T+1})) | \mathcal{F}_{T})].
\end{equation}

There exists a detector $\Delta_\alpha^\star$, an encoding function $\phi$, and a constant $c$ such that the following result holds in the limit as $\alpha \rightarrow 0$. If $(\phi(X_1), \dots, \phi(X_T), \phi(X_{T+1}), \dots)$ are $(\Delta_\alpha^\star, \epsilon)$-practically iid, then:
\begin{equation}
    \text{kl}( l_\text{test} \| l_\text{train}) \leq c\epsilon,
\end{equation}
where $\text{kl}(\cdot\|\cdot)$ is the KL-divergence between two Bernoulli random variables with parameters $l_\text{test}$ and $l_\text{train}$. 
\end{proposition}
\begin{proof}
    Define $\phi: x \mapsto l(x, h(x))$. The random variables $(\phi(X_1), \dots, \phi(X_{T+1}), \dots)$ then correspond to Bernoulli random variables with dependent, time-varying means. In the limit $\alpha \rightarrow 0$, the detector presented in \cite{shin2022detectors} has FAR bounded by $\alpha$ and achieves a WAD  $ \leq c \log(1/\alpha) / \text{kl}(l_\text{test} \| l_\text{train}))$. Now, suppose for contradiction that $\text{kl}( l_\text{test} \| l_\text{train}) > c\epsilon$. Then, we have WAD $< (1/\epsilon) \log(1/\alpha)$, which contradicts the statement that $(\phi(X_1), \dots, \phi(X_T), \dots)$ are $(\Delta_\alpha^\star, \epsilon)$-practcially iid. 
\end{proof} 

The practical utility of the detector $\Delta_\alpha^\star$ above is limited since it takes losses as input; because we ultimately rely only on making training data practically iid, $\Delta_\alpha^\star$ can be deceived into not detecting a distribution shift on training data simply by overfitting and driving the loss on all training examples to 0. To address this challenge, we allow for detectors (e.g., based on conformal martingales) that take latent representations $\phi(x) \in \mathbb{R}^d$ as input. The following corollary follows immediately from Proposition~\ref{prop:generalization}.

\begin{corollary}
Let $h_\phi$ be a hypothesis with latent encoding $\phi$. Consider a detector $\Delta_\alpha$ that observes inputs encoded by $\phi$, and that is at least as efficient as the detector $\Delta_\alpha^\star$ that relies on loss values, i.e., the WAD of $\Delta_\alpha$ for any pre- and post-change distributions is less than or equal to the WAD of the detector $\Delta_\alpha^\star$. Then, there exists a constant $c$ such that the following result holds in the limit as $\alpha \rightarrow 0$. If $(\phi(X_1), \dots, \phi(X_T), \phi(X_{T+1}), \dots)$ are $(\Delta_\alpha, \epsilon)$-practically iid, then
$\text{kl}( l_\text{test} \| l_\text{train}) \leq c\epsilon$. 
\end{corollary}

The results above rely on having access to test data. Instead, consider a representation $\phi$ such that the training data sequence $(\phi(x_1), \dots, \phi(x_T))$ is $\Delta$-practically iid, and define the $\Delta$-span of this representation as containing test distributions such that $(\phi(X_1), \dots, \phi(X_T), \phi(X_{T+1}), \dots)$ are $(\Delta_\alpha, \epsilon)$-practically iid. Then, in the limit as $\alpha \rightarrow 0$, it follows from the results above that for any test distribution in the $\Delta$-span, $\text{kl}( l_\text{test} \| l_\text{train}) \leq c\epsilon$. 

}


\end{document}